\documentclass{article}

\usepackage{microtype}
\usepackage{graphicx}
\usepackage{subfigure}
\usepackage{booktabs} % for professional tables
\usepackage[]{amsmath,amssymb}
\usepackage{amsthm}
\usepackage{enumerate}

\usepackage[hyphens]{url}
\usepackage{hyperref}
\hypersetup{breaklinks=true}

\newtheorem{theorem}{Theorem}
\newtheorem{proposition}{Proposition}
\newtheorem{corollary}{Corollary}

\usepackage[accepted]{icml2021}

\icmltitlerunning{A Probabilistic Interpretation of Transformers}

\begin{document}

\twocolumn[
\icmltitle{A Probabilistic Interpretation of Transformers\\
           International Conference on Machine Learning (ICML 2021)}



\icmlsetsymbol{}{}

\begin{icmlauthorlist}
\icmlauthor{Alexander Shim}{ml}
\end{icmlauthorlist}

\icmlaffiliation{ml}{ML Collective}

\icmlcorrespondingauthor{Alexander Shim}{alex.shim@gmail.com}

\icmlkeywords{Transformers, Attention, Contrastive Learning, Exponential Families}

\vskip 0.3in
]



\printAffiliationsAndNotice{}  % leave blank if no need to mention equal contribution

\begin{abstract}
We propose a probabilistic interpretation of exponential dot product attention of transformers and contrastive learning based off of exponential families. The attention sublayer of transformers is equivalent to a gradient ascent step of the log normalizer, which is the log-sum-exp term in the Hopfield theory of attention. This ascent step induces a parallel expansion of points, which is counterbalanced by a contraction from layer normalization. We also state theoretical limitations of our theory and the Hopfield theory and suggest directions for resolution.
\end{abstract}

\section{Introduction}
\label{introduction}

Transformers have reached state of the art results in language models, significantly outperforming LSTMs. One conceptual explanation for their greater performance is the ability of attention to utilize long range dependencies, whereas Recurrent Neural Networks are limited to encoding past information within a fixed-size hidden state. 
What this explanation does not explain is how certain architectural choices of transformers, specifically exponential dot product attention, also somewhat ambiguously referred to as softmax attention, outperforms alternatives.

Exponential dot product attention has been popularized in contrastive learning and metric learning and is used in state of the art semi-supervised contrastive learning models.
In language models, an exponential dot product probability is used to model conditional probabilities in Word2Vec as well as in some memory networks. 

The successes of transformers has been verified empirically, but little work
has focused upon a theoretical framework for transformers. 
We offer a probabilistic explanation, based off of distributions of the exponential
family, for attention and contrastive probabilities. 
Expressing attention as an exponential family allows us to utilize related theory in statistics, machine learning, and statistical mechanics, offering insightful interpretations of the transformer architecture. We provide proofs for attention updates over several continuous distributions as well.

We also explicitly state the limitations of our theory, noting that the modern
Hopfield network interpretation of attention shares many of these limitations. For
some of these limitations, we speculate connections between other areas of re-
search which may reconcile the theoretical inconsistencies, motivating directions
for future research.

\section{Background}
\label{background}

\subsection{Exponential Dot Product Attention}
Word2vec used a skip-gram model to predict neighboring words using a conditional distribution defined by a normalized exponential dot product multinoulli function \cite{DBLP:journals/corr/MikolovSCCD13}

Attention was proposed through normalized exponential alignment functions, often referred to as softmax attention in literature, for Neural Machine Translation \cite{DBLP:journals/corr/Graves13,bahdanau2014neural}, and later work parallelizing computation on sequential data introduced normalized exponential dot product similarity \cite{DBLP:journals/corr/ParikhT0U16} (A Decomposable Attention Model for Natural Language Inference).
Neural Turing Machines gated memory updates using normalized exponential cosine similarity, in what is referred to as soft attention \cite{DBLP:journals/corr/GravesWD14}.

Other transformer precursors parallelized attention updates over the entire sequence into a layer and switched to convolution-based attention weights \cite{kaiser2016neural,DBLP:journals/corr/KaiserB16}. 
The transformer architecture incorporated exponential dot product attention scaled by dimensionality \cite{vaswani2017attention}. 

\subsection{Contrastive Learning and Metric learning}
\label{contrastive}
Noise-Contrastive Estimation (NCE) creates a mixture distribution between real data and noisy data to convert an unsupervised learning problem into a semi-supervised learning problem, modeling the contrastive probabilities as a parametrized logistic distribution \cite{Gutmann2010}.

In metric learning, Multidimensional scaling calculates pairwise distances between projected points \cite{cox2000multidimensional}, which for Euclidean distances are equivalent to calculating covariance matrix terms using dot products \cite{bishop2007}.
Neighborhood Components Analysis learns a low dimensional linear embedding matrix and models a probability of a neighbor by comparing the exponential negative square distance
$
  e^{ - \lVert A x_i - A x_j \rVert^2 }
$
of two input data $x_i,x_j$ to the sum of the exponential negative distances to non-neighbors \cite{conf/nips/GoldbergerRHS04}

Due to slow convergence of Bernoulli contrastive loss and triplet loss, Sohn proposed an exponential dot product probability over multiple examples \cite{Sohn2016ImprovedDM}, which is mathematically consistent with NCE for multiple distributions. The paper’s roots in metric learning motivated the dot product form, with a direct influence from Neighborhood Component Analysis.

More recent contrastive learning research adopted a contrastive loss based off of exponential dot product probabilities, including papers that achieve state of the art semi-supervised learning \cite{wu2018unsupervised,DBLP:journals/corr/abs-2002-05709}.

\subsection{Shortcut Connections and Dynamical Systems}
\label{residual connections}

Long Short-Term Memory (LSTM)  combined a shortcut connection to deal with the vanishing and exploding gradient problem along with gating functions to incorporate and forget information \cite{HochSchm97}. Residual connections similarly formulated the hidden layer as an update to an identity mapping, though without a gating mechanism \cite{he2015deep}. Recurrent Neural Networks have been interpreted as a discrete time approximation to a continuous dynamical system \cite{article}, where gating acts as a warping of time \cite{tallec2018recurrent}. Residual connections have been interpreted as a discretized update to a differential equation \cite{Weinan2017APO,lu2020finite}.

Interpreting residual networks as discretized differential equations, researchers have posed alternative methods for performing forward updates to converge to equilibrium points and backwards updates to the parameters from the equilibrium points \cite{chen2019neural,bai2019deep}. Further work has used monotone operator theory in convex analysis for solving for equilibrium points, interpreting layers as an operator \cite{winston2021monotone}.

Mathematically similar to our work, transformers have been interpreted as an update of modern hopfield networks and fixed points have been calculated with respect to a fixed set of patterns \cite{DBLP:journals/corr/abs-2008-02217}.  Our work similarly views the attention sublayer as an operator update over a class of discretized probability distribution, though with a changing set of patterns.

\subsection{Log Normalizer and Free Energy}
\label{log normalizer}

Partition functions, or the normalizer function, in statistical physics defines a normalization factor of the Hamiltonian with respect to a parameter defining the temperature. The Boltzmann distribution can be derived through Lagrange multipliers as the distribution which maximizes entropy with a conservation of energy constraint. Jaynes adapted the Boltzmann distributions to maximum entropy distributions with multiple expected statistics constraints by converting the maximum entropy problem into the dual problem of optimizing the log normalizer \cite{1456693}, which is known in statistical mechanics as free energy.

Variational methods have been used to approximate log probabilities of observations in machine learning, borrowing from ideas in statistical mechanics. By viewing the joint as an unnormalized probability distribution, the log normalizer is known as the evidence lower-bound, and it has connections to Helmholtz Free Energy \cite{Hinton:1995,koller2009probabilistic}.

The sum of exponents loss of AdaBoost \cite{collins00logistic} has been interpreted as the dual form of generalized KL divergence. The log sum of exponents is well known in convex optimization to be the dual form to the maximum entropy objective for a discrete probability distribution \cite{citeulike:163662}. Notably, in the modern Hopfield network interpretation of transformers the log sum of exponents is used as part of the energy function and solved through convex optimization techniques \cite{DBLP:journals/corr/abs-2008-02217}.

In the work most similar to ours, McBal suggested an energy-based model interpretation of transformers, with separate energy updates for each attention sublayer and fully-connected layer, as well as additional ideas from statistical physics \cite{bal_2020}.

\section{Exponential Dot Product Attention}
\label{exponential dot product attention}

\subsection{Exponential Families}
\label{exponential families}
The natural parameter form, also known as the canonical form or a linear exponential family, of a distribution of the exponential family can be written as
\begin{equation}
  p(x \vert \eta) = 
    \frac{1}{Z \left( \eta \right) }
    h(x)
    e^{ u(x)^\intercal \eta}
  \label{exponential family}
\end{equation}
where $x$ is the random variable, 
$u(x)$ the sufficient statistic, 
$\eta$ the natural parameter, 
$h(x)$ the intrinsic measure or carrier measure, 
and 
$Z ( \eta ) := \int h(x) e^{ u(x)^\intercal \eta }$ the normalizer or partition function.
An exponential family distribution can be generated from an intrinsic measure $h(x)$ by defining the sufficient statistic $u(x) := x$, and defining the normalizer as the Laplace transform applied to the intrinsic measure:
\begin{equation}
  Z ( \eta ) = \mathcal{L} \circ h(x) = \int h(x) e^{ x^\intercal \eta} dx
  \label{Laplace transform}
\end{equation}
% The unnormalized probability mass or probability density may be written as $\tilde{p}(x \vert \eta)$.

When $h(x)$ is chosen to be a uniform discrete measure over a finite set of points 
$\{x_n\}_{n=1}^N $
in a continuous space, the probability density converts to a probability mass function, and the normalizer is the summation 
$Z( \eta ) = \sum_{n=1}^N e^{ u(x_n)^\intercal \eta} $
instead of an integral. 

When the query $q_i$ is chosen as the natural parameter and the keys as the finite set of points $\{ k_j \}_{j=1}^N$, exponential dot product attention weights are of the form of an exponential family
\begin{equation}
  p( k_j \vert q_i ) = 
  \frac
    { e^{k_j \cdot q_i } }
    { Z(q_i) }
  \label{exponential family attention}
\end{equation}

The expected sufficient statistic of an exponential family can be written as a one-to-one function of the natural parameter
\begin{equation}
  E_{ x \sim P( x \vert \eta) } \left[ u(x) \right] =
  \nabla_\eta \log Z(\eta)
  \label{activation function}
\end{equation}
For the exponential family defined by attention, we observe that attention averaging is exactly the gradient of the log normalizer
% \begin{equation}
%  \nabla_\eta \log Z(\eta) \vert_{\eta = q_i} = 
%  \sum_{j=1}^N 
%    \frac
%      { e^{ k_j \cdot q } }
%      { \sum_{n'=1}^N e^{k_{n'} \cdot q_i} }
%    k_j
%  \label{attention gradient update}
% \end{equation}

We further prove results for attention updates of non-discrete $h(x)$. Proofs and further results are given in the appendix.

\begin{proposition}
  Let $X = R^D$, $h: X \rightarrow R^{+}$. 
  \begin{enumerate}[(a)]
    \item If $h(x) := \sum_{n=1}^N \delta( x = x_n)$, then \\
    $
      \nabla_\eta \log \int h(x) e^{ x^\intercal \eta } dx 
      = \sum_{n=1}^N \frac{ e^{ x_n^\intercal \eta } }{ \sum_{n'} e^{ x_{n'}^\intercal \eta } } x_n 
    $.
    \item If $h(x) = p_0( x \vert \eta_1, \eta_2)$, where $p_0(x \vert \eta_1,\eta_2)$ is the exponential family distribution $p_0(x \vert \eta_1,\eta_2) = \frac{1}{Z_0(\eta}) h_0(x) e^{x^\intercal \eta} e^{u_2(x)^\intercal \eta_2}$, with sufficient statistic $u_1(x) = x$ and arbitrary sufficient statistic $u_2(x)$, natural parameters $(\eta_1,\eta_2)$, intrinsic measure $h_0(x)$, and normalizer $Z_0(\eta_1,\eta_2) = \int h_0(x) e^{ x^\intercal \eta_1 } e^{ u_2(x)^\intercal \eta_2 } dx$, then $\nabla_{_\eta} \log \int h(x) e^{ x^\intercal \eta} = E_{x \sim p_0(x \vert \eta_1 + \eta, \eta_2)} \left[ x \right]$
  \end{enumerate}
  \label{exponential family gradient updates}
\end{proposition}

\begin{corollary}
  If $h(x) = \mathcal{N}( x; \mu,\Sigma)$, then \\
  $\nabla_\eta \log \int h(x) e^{ x^\intercal \eta} dx = \mu + \Sigma \eta$.
  \label{cor gaussian update}
\end{corollary}

When the attention sublayer is added to the residual connection, we observe that we are performing a gradient update of the log normalizer with respect to the natural parameters, which are the hidden states.

The log normalizer for this discrete distribution is the log sum of exponents, which is a component of the modern Hopfield energy function, where the attention sublayer also acts as an update for the hidden states
\begin{equation}
  \log Z(q_i) = \log \sum_{j=1}^N e^{ k_j \cdot q_i}
  \label{free energy}
\end{equation}

\subsection{Log normalizer of exponential families}

The set of queries can be seen as IID samples from a distribution of natural parameters \footnote{McBal refers to attention as a mean-field approximation}, meaning our gradient operator is updating a distribution of distributions. In information geometry, natural parameters are naturally mapped to their exponential family distributions, specifically when the Bregman divergence $B_G(\eta_1,\eta_2)$ is defined by the log normalizer $G(\eta) = \log Z(\eta)$, equaling $KL \left( p(x \vert \eta_1) \Vert p(x \vert \eta_2) \right)$. 

The Bregman divergence of the convex conjugate of the log normalizer is also equivalent to KL divergence, and the dual variable is exactly the expected sufficient statistic $\bar{u}$, relating back to the activation function \eqref{activation function}.

\begin{equation}
  \begin{split}
    G^*(\eta^*) = \sup_\eta \left( \eta^\intercal \eta^* - G(\eta) \right)\\
    \eta^* = \bar{u} = \nabla_\eta \log Z(\eta)
  \end{split}
  \label{Fenchel duality}
\end{equation}

$G^*( \bar{u} ) = -H( p( x \vert \bar{u} ) - E_{ x \sim p( x \vert \bar{u} ) } \left[ \log h(x) \right]$ is an offset of the negative entropy function acting on the exponential family distribution with expected sufficient statistic $\bar{u}$, meaning the dual problem of attention is a variant of the maximum entropy problem.
More importantly, the duality between natural parameters and expected sufficient statistics can be used to transform a distribution of natural parameters, that is, our query space, into a distribution of expected sufficient statistics, that is, our key space, through our activation function. Unfortunately, the activation function is not always an affine function, but if we attempt an affine approximation, we recover $\nabla^2 \log Z(\eta) = E_{ x \sim p(x)} \left[ u(x) u(x)^\intercal \right]$, the Fisher information matrix.

When our hidden states is our natural parameter distribution, our affine approximation for the key corresponding to $h_j$ is $\Sigma^{-1} h_j$, exponential dot product attention is proportional to $ e^{h_i^T \Sigma^{-1} h_j} $, suggesting our key and query weights $W_q^\intercal W_k \approx \Sigma^{-1}$.

\subsection{Expansion and contraction}
\label{expansion and contraction}

The attention sublayer outputs a convex combination of the keys, $\sum_{n=1}^N A(k_n,q) k_n$. Without a residual connection, repeated applications of attention would contract the hidden states into the interior, intuitively risking a collapse to a single fixed point.

With a residual connection, assuming our hidden states are recentered around the origin through layer normalization or some other normalization, we could intuitively imagine that roughly radially symmetric hidden states would push each hidden state further away from the origin.
Since the log normalizer is the dual form of the maximum entropy distribution, our log normalizer ascent should result in increased entropy, resulting in an expansion of the hidden states away from their starting points.

We prove that at equilibrium this result holds for uncentered Gaussian distributions. The more formal statement and proof are given in the appendix.

\begin{theorem}
  (Informal) $p_{eq}(\eta) = \mathcal{N}( \Sigma^{-1} \mu, \Sigma^{-1} )$ is an equilibrium distribution of the renormalization operator $RN$ with mean $\Sigma^{-1} \mu$ and covariance $\Sigma^{-1}$,  composed with the attention update operator, assuming intrinsic measure $h(x): \eta \rightarrow \Sigma \eta$.
  \label{informal theorem equilibrium}
\end{theorem}

\section{Operator Interpretation}
\label{operator interpretation}

Starting with a hidden state of $N$ tokens, our attention heads are simultaneous gradient updates on each token. 
We can interpret the tokens as a discretization of a distribution, drawing parallels to the empirical distribution defined by sampling, with the attention operator acting as a pointwise gradient update of the distribution of natural parameters.
Two interesting cases of deriving an exponential family from an intrinsic measure are when a conjugate prior $ p( \eta \vert \xi, \nu )$ is defined as the intrinsic measure, and in boosting where the previous model defines an intrinsic measure of unnormalized weights by $w_i = e^{ - t_i \eta_n(x_i) } $, where $\eta_n$ is the previous iteration of the model logits \cite{collins00logistic}.

The FC layer can be viewed as a discrete approximation to an operator as well, so the entire attention block can be viewed as a distributional operator.  If different attention blocks have tied weights, then it is the same operator applied repeatedly, otherwise the operators for each layer change.

\section{Theoretical Limitations}
\label{limitations}
Both this work and the Hopfield interpretation of transformers initially ignore the weight matrices, which are the only parameters that the network learns for the attention sublayer. 
With transformer models having hundreds of millions to a billion parameters and being difficult to train, this oversight is deeply problematic. 
The simplest interpretation is for the weight matrices to represent different distortions of space, in which case different attention heads would represent different distortions.
Unfortunately, this makes multihead attention incoherent, as we would be combining gradient updates from different spatial transformations. We present a somewhat more consistent lower bound approximation of the log normalizer of a mixture of Gaussians in the appendix.

For this work, we have interpreted a single update on a single hidden state as a gradient update. 
However, it is less clear why each hidden state has a separate gradient update, and since they depend on the other hidden states, they are not independent updates. The Ising model, which motivated the original Hopfield network, has similar properties. Sample-based approximate inference techniques for Sequential Monte Carlo, particularly ones utilizing bootstrapping, are dependent on the samples and qualitatively similarly use $N$ samples to generate $N$ new samples for the next time step.

The FC layer is inconsistent with both the Hopfield theory of attention and this paper's gradient update of the log normalizer. We would have to view it as a separate operator update composed on top of our attention operator.

\section{Future Work}
\label{future work}
We have provided a foundation of a probabilistic interpretation, and we specifically interested in utilizing ideas of statistical sampling and Sequential Monte Carlo over hybrid distributions. 

The Hopfield theory of transformers solved for equilibrium behavior, but specifically when the patterns were fixed, whereas for transformers the patterns are the dynamically changing hidden states.
A useful experiment would be to sample from an initial distribution and observe how distributions change with each layer, especially with learned weights.

Since the theoretical interpretations ignore the FC layer, it would be useful to know how close the FC layer is to an identity mapping or a purely linear layer, if at all.

Assuming layer normalization should induce an equilibrium distribution of a Gaussian distribution, we could potentially try other contractive mappings to see if this generates substantially different embeddings and layer behavior.

One key difference in theories is that in this work, the residual connection is used as an initial point for a gradient update, whereas the Hopfield theory of attention defines a quadratic term in the Hopfield energy, strongly suggesting that we are working with Gaussian distributions. 
Since exponential dot product probabilities may come from locally Euclidean assumptions in metric learning, weighted attention may be making inherent assumptions about Gaussian distributions or mixtures of Gaussians over the key space.

Non-Gaussian exponential families will not map the natural parameter to the expected sufficient statistic through an affine transformation. In those cases, we would need to apply a non-linear activation function to transform the natural parameter space into the sufficient statistic space before we take the dot product. An affine approximation of the activation function is the Fisher information matrix, so we could test to see if the query and key weight matrices combine to an approximation of the Fisher information matrix.

\bibliography{prob_transformers}
\bibliographystyle{icml2021}

\end{document}

\onecolumn[
\icmltitle{A Probabilistic Interpretation of Transformers}



\icmlsetsymbol{}{}

\begin{icmlauthorlist}
\icmlauthor{Alexander Shim}{ml}
\end{icmlauthorlist}

\icmlaffiliation{ml}{ML Collective}

\icmlcorrespondingauthor{Alexander Shim}{alex.shim@gmail.com}

\icmlkeywords{Machine Learning, ICML}

\vskip 0.3in
]



\printAffiliationsAndNotice{}  % leave blank if no need to mention equal contribution

\bibliography{example_paper}
\bibliographystyle{icml2021}

\appendix

\section{Proofs}
\label{proofs}

\begin{proposition}
  Let $X = R^D$, $h: X \rightarrow R^{+}$. 
  \begin{enumerate}[(a)]
    \item If $h(x) := \sum_{n=1}^N \delta( x = x_n)$, then 
    $
      \nabla_\eta \log \int h(x) e^{ x^\intercal \eta } dx 
      = \sum_{n=1}^N \frac{ e^{ x_n^\intercal \eta } }{ \sum_{n'} e^{ x_{n'}^\intercal \eta } } x_n 
    $.
    \item If $h(x) = p_0( x \vert \eta_1, \eta_2)$, where $p_0(x \vert \eta_1,\eta_2)$ is the exponential family distribution $p_0(x \vert \eta_1,\eta_2) = \frac{1}{Z_0(\eta}) h_0(x) e^{x^\intercal \eta} e^{u_2(x)^\intercal \eta_2}$, with sufficient statistic $u_1(x) = x$ and arbitrary sufficient statistic $u_2(x)$, natural parameters $(\eta_1,\eta_2)$, intrinsic measure $h_0(x)$, and normalizer $Z_0(\eta_1,\eta_2) = \int h_0(x) e^{ x^\intercal \eta_1 } e^{ u_2(x)^\intercal \eta_2 } dx$, then $\nabla_{_\eta} \log \int h(x) e^{ x^\intercal \eta} = E_{x \sim p_0(x \vert \eta_1 + \eta, \eta_2)} \left[ x \right]$
  \end{enumerate}
  \label{exponential family gradient updates}
\end{proposition}

\begin{proof}
  \begin{enumerate}[(a)]
    \item $Pr \{ x = x_n \} = \frac{ e^{ x_n^\intercal \eta} }{ \sum_{n'} e^{x_{n'}^\intercal \eta } }$, so $\nabla_\eta \log Z(\eta) = E_{x \sim p( x \vert \eta) } \left[ x \right] = \sum_n \frac{ e^{x_n^\intercal \eta } }{ \sum_{n'} e^{x_{n'}^\intercal \eta } } x_n $. \\
    \item We collect exponential terms into a distribution of the same exponential family as $p_0(x \vert \eta_1,\eta_2)$

  \begin{equation}
    \begin{split}
      \int p_0(x) e^{ x^\intercal \eta} dx 
        &= \frac{1}{Z_0(\eta_1,\eta_2)} \int h_0(x) e^{ x^\intercal ( \eta_1 + \eta_2) } e^{u_2(x)^\intercal \eta_2} dx\\
      &= \frac{Z_0(\eta_1+\eta,\eta_2)}{Z_0(\eta_1,\eta_2)} \\
    \end{split}
    \label{}
  \end{equation}

  When evaluating the score, the denominator term has no dependence on $\eta$

  \begin{equation}
    \begin{split}
      \nabla_\eta \log \int p(x \vert \eta_1,\eta_2) e^{ x^\intercal \eta)} dx
	&= \nabla_\eta \log \frac{Z_0(\eta_1+\eta),\eta_2}{Z_0(\eta_1,\eta_2)} \\
      &= \nabla_\eta \log Z_0( \eta_1 + \eta,\eta_2) \\
      &= E_{ x \sim p( x \vert \eta_1 + \eta, \eta_2)} \left[ x \right]
    \end{split}
  \end{equation} 
  \end{enumerate}
\end{proof}

\begin{corollary}
  If $h(x) = \mathcal{N}( x; \mu,\Sigma)$, then $\nabla_\eta \log \int h(x) e^{ x^\intercal \eta} dx = \mu + \Sigma \eta$.
  \label{cor gaussian update}
\end{corollary}

\begin{proof}
  The natural parameter form of a Gaussian is $\frac{1}{Z(\eta_1,\eta_2)} e^{ \left< x, \Sigma^{-1} \mu \right> + \left< x x^\intercal, - \frac{\Sigma^{-1}}{2}\right>}$, where $\left< \, . \, , \, . \, \right>$ denotes the vectorized dot product for both vectors and matrices, with $\eta_1 = \Sigma^{-1} \mu$. Conversely, $\mu = \Sigma \eta_1$. \\
  Hence, $E_{ x \sim p( x \vert \Sigma^{-1} \mu + \eta, - \frac{\Sigma^{-1}}{2}) } \left[ x \right] = \mu + \Sigma \eta$
\end{proof}

\begin{proposition}
  If $h(x) = \sum_{n=1}^N \pi_n \mathcal{N}(\mu_n,\Sigma )$, where $\pi_n \in R^{+}$, then $\nabla_\eta \log \int h(x) e^{ x^\intercal \eta } dx = \Sigma \eta + \sum_{n=1}^N \frac{ \pi_n e^{\mu_n^\intercal \eta} }{ \sum_{n'=1}^N \pi_{n'} e^{ \mu_{n'}^\intercal \eta} } \mu_n$
  \label{rbf update}
\end{proposition}

\begin{proof}
  For a Gaussian $\mathcal{N}(x;\mu,\Sigma)$ for $x \in R^d$, the exponential family normalizer is $(2 \pi)^{\frac{D}{2}} \lvert \Sigma \rvert ^{\frac{1}{2}} e^{\frac{1}{2} \mu^\intercal \Sigma^{-1} \mu }$. \\
  \begin{equation}
    \begin{split}
      \int \sum_{n=1}^N \pi_n \mathcal{N}(x;\mu_n,\Sigma) e^{x^\intercal \eta} dx &= \sum_{n=1}^N \frac{1}{Z(\mu_n,\Sigma)} \pi_n \int e^{ \left< x x^\intercal, - \frac{1}{2} \Sigma^{-1} \right> + \left<x, \Sigma^{-1} \mu_n \right> + x^\intercal \eta} dx \\
      &= \sum_{n=1}^N \pi_n \frac{Z(\mu_n + \Sigma \eta,\Sigma)}{Z(\mu_n,\Sigma)} \\
      &= \sum_{n=1}^N \pi_n e^{\frac{1}{2} (\mu_n + \Sigma \eta)^\intercal \Sigma^{-1} (\mu_n + \Sigma \eta) - \frac{1}{2} \mu_n^\intercal \Sigma^{-1} \mu_n} \\
      &= e^{\frac{1}{2} \eta^\intercal \Sigma \eta} \sum_{n=1}^N \pi_n e^{\mu_n^\intercal \eta}
    \end{split}
    \label{rbf normalizer}
  \end{equation}

  The gradient of the log normalizer simplifies to

  \begin{equation}
    \begin{split}
      \nabla_\eta \log Z(\eta) = \Sigma \eta + \nabla_\eta \log \sum_{n=1}^N \frac{ \pi_n e^{\mu_n^\intercal \eta} }{ \sum_{n'=1}^N \pi_{n'} e^{\mu_{n'}^\intercal \eta }} \mu_n
    \end{split}
    \label{rbf update eq}
  \end{equation}

  which when $\pi_n$ are uniform $\forall n$ is the sum of the discrete attention update and the $\Sigma \eta$ term.
\end{proof}

If our intrinsic measure is a mixture of Gaussians with different covariance matrices, then we can perform a gradient update on a lower bound, $G_{LB}(\eta)$, of the log normalizer.

\begin{proposition}
  If $h(x) = \sum_{n=1}^N \pi_n \mathcal{N} \left( \mu_n, \Sigma_n \right)$, then there exists a lower bound $G_{LB} = \sum_{n=1}^N \pi_n \left( \frac{1}{2} \eta^\intercal \Sigma \eta + \mu_n^\intercal \eta \right) \leq \log Z(\eta)$, with gradient update $\nabla_\eta G_{LB}(\eta) = \sum_{n=1}^N \pi_n ( \mu_n + \Sigma_n \eta )$
  \label{}
\end{proposition}

\begin{proof}
  Using Jensen's inequality on the concave logarithm function,

  \begin{equation}
    \log \int \sum_{n=1}^N \pi_n \mathcal{N} \left( x; \mu_n, \Sigma_n \right) e^{ x^\intercal \eta} dx \geq \sum_{n=1}^N \pi_n \log \int \mathcal{N} ( x; \mu_n, \Sigma_n ) e^{ x^\intercal \eta} dx
    \label{Jensen's inequality}
  \end{equation}

  From \eqref{rbf update eq}, 

  \begin{equation}
    = \sum_{n=1}^N \pi_n \left( \frac{1}{2} \eta^\intercal \Sigma_n \eta + \mu_n^\intercal \eta \right)
    \label{lower bound}
  \end{equation}

  If we perform a gradient update on the lower bound, we have
  
  \begin{equation}
    \nabla_\eta G_{LB}(\eta) = \sum_{n=1}^N \pi_n ( \mu_n + \Sigma_n \eta )
    \label{lower bound update}
  \end{equation}

\end{proof}

Using Corollary $\autoref{cor gaussian update}$, we can apply a generalization of the attention sublayer to update a distribution $p(\eta)$ of natural parameters given a Gaussian intrinsic measure, $h(x) = \mathcal{N}(x;\mu,\Sigma)$. Moreover, if we assume there is a one-to-one transformation of the natural parameters into the intrinsic measure through the activation function, we can prove the stationarity of the attention update operator composed with a renormalization:

\begin{theorem}
  Suppose the distribution of natural parameters is $p_{eq}(\eta) = \mathcal{N}( \Sigma^{-1} \mu, \Sigma^{-1} )$ and we define the intrinsic measure $h(x)$ by a pointwise transformation $ x = \Sigma \eta$. Let $RN_{\eta_0,\Lambda_0}$ be the renormalization operator such that if $E_{ \eta \sim p(\eta) } \left[ ( \eta, \eta \eta^\intercal ) \right] = (\bar{\eta}, \bar{\Sigma}) $, $RN_{\eta_0,\Lambda_0}$ maps $\eta \rightarrow \eta_0 + \Lambda_0^{\frac{1}{2}}\bar{\Sigma}^{-\frac{1}{2}} (\eta - \bar{\eta})$. 
  Then the composition of renormalization operator and the attention update operator $RN_{\Sigma^{-1}\mu,\Sigma^{-1}} \circ A \circ ( p_{eq}(\eta), h(x) ) = p_{eq}(\eta)$ . 
  \label{theorem equilibrium}
\end{theorem}

\begin{proof}
  The attention operator acts pointwise on each $\eta$ by $\eta' = \eta + \nabla_\eta \log \int h(x) e^{x^\intercal \eta} dx$. We can reparametrize $\eta$ as $ \eta = \Sigma^{-1} \mu + \Sigma^{ - \frac{1}{2} } \epsilon$, where $\epsilon \sim \mathcal{N}(0,I)$. Hence $x = \mu + \Sigma^{\frac{1}{2}} \epsilon$, which is the reparametrization of $\mathcal{N}(\mu,\Sigma)$. \\
  By Corollary $\autoref{cor gaussian update}$,
  \begin{equation}
    \eta' = \mu + ( I + \Sigma) \eta
    \label{gaussian attention update}
  \end{equation}
  This is an affine transformation of the natural parameter, and an affine transformation of a gaussian random variable is a gaussian random variable. The renormalization operator is another affine transformation which remaps it to a distribution with mean and covariance of $p_{eq}(\eta)$, and once again it must be a Gaussian distribution.

\end{proof}

\section{Hopfield Networks is All You Need comparison}
\label{hniayn}

One key difference is that for the Hopfield theory, the patterns are fixed, while for our theory the patterns are the changing hidden states. Notably later versions of the Hopfield paper added Hopfield layers with changing patterns, but theoretical proofs of the limiting behavior apply for fixed patterns. Arguably, the Hopfield analysis was not meant to answer how the attention sublayers transformed the input, and the focus on exponential storage capacity focused more on the ability of attention to quickly converge from noisy inputs to stored patterns.

Notably, while the encoder and decoder involve attention sublayers which use the hidden states as patterns, fitting our interpretation better, the decoder layers also add a second attention block which apply attention to the encoder input, fitting the Hopfield paper conditions.

From a conceptual perspective, the advantage of the original Hopfield network was that the patterns were stored in the weights. For modern Hopfield networks, the patterns must be stored separately, and they are not directly related to the transformer weights, removing the benefit of biological plausibility. 

The Hopfield theory requires a more specific set of assumptions, which justify the energy function, whereas our theory is based off of information theoretic ideas of a widely used free energy objective, which is dual to maximum entropy. We argue that our theory is more directly responsible for the theoretic interpretation of the Hopfield energy function and gradient updates and it does not rely on motivating factors behind Hopfield networks.

One mutual deficiency for both theories is that they do not explain the transformer attention weight transformations nor does it explain multihead attention. Not having to consider weight transformations, multihead attention, or FC layers makes equilibrium analysis much simpler. Unfortunately, the only parameters in the attention sublayer are in these weight transformations, requiring all knowledge of the data distribution to be encoded into these parameters - otherwise we might expect our language models to perform fairly well without any weights. However, it may be possible that some weights may be less sensitive or at least easier to train than others - linearly transforming the key and query spacies for the attention probabilities may focus on certain subspaces more than others, potentially leading to more efficient convergence, yet perhaps the Hopfield benefit of exponentially fast convergence may still apply regardless. In contrast, the value weights combine with the multihead attention linear transformation and the fully connected layer, leading to a much more nonconvex problem and training instability.

Dot product approximations for exponential dot product attention have been used, and perhaps those works have some connection to the original Hopfield network Ising model interpretation. Recent work \cite{} has further explored the connection between Hopfield networks and Restricted Boltzmann Machines. Perhaps the FC layer or linear attention can be seen as original Hopfield network layers and exponential dot product attention modern Hopfield network layers.

The Hopfield theory seems to attempt continuous integration of states, which requires a quadratic term in the energy to keep the partition function finite. The quadratic term also seems to fulfill another duty of modeling the skip connection. While we perform a discrete normalization, it is reasonable to assume that our set of hidden states are equivalent to hidden states from a continuous distribution. What this distribution is unclear, though it draws parallels to Arora's random walk of language around a context vector, which may be equivalent to McBal's mean-field approximation interpretation of transformers.